\documentclass{article}
\usepackage{ijcai20}

\newcommand{\alg}{\ensuremath{\mathsf{AttenAlg}}\xspace}
\newcommand{\lpalg}{\ensuremath{\mathsf{WarmUp}}\xspace}
\newcommand{\sr}{\ensuremath{\mathsf{SR}}\xspace}
\newcommand{\pro}{\ensuremath{\mathsf{Profit}}\xspace}
\def \OPTP {\ensuremath{\operatorname{OPT-P}}\xspace}
\def \OPTF {\ensuremath{\operatorname{OPT-F}}\xspace}
\def \ALGP {\ensuremath{\operatorname{ALG-P}}\xspace}
\def \ALGF {\ensuremath{\operatorname{ALG-F}}\xspace}

\usepackage{macros_pan}

\usepackage{times}  
\usepackage{soul}
\usepackage{url}  
\usepackage[hidelinks]{hyperref}
\hypersetup{colorlinks, citecolor=red, filecolor=blue, linkcolor=blue, urlcolor=blue}
\usepackage{cleveref}
\usepackage[utf8]{inputenc}
\usepackage[small]{caption}
\usepackage{graphicx}  
\usepackage{subfigure}
\usepackage{amsmath}
\usepackage{booktabs}
\usepackage{xcolor}
\usepackage{xspace}

\title{Trading the System Efficiency for the Income Equality of Drivers in Rideshare\thanks{Copyright \copyright~2020 International Joint Conferences on Artificial
Intelligence (IJCAI). All rights reserved.}}
\author{
Yifan Xu$^1$
\and
Pan Xu$^2$
\affiliations
$^1$Key Lab of CNII, MOE, Southeast University, Nanjing, China\\
$^2$Department of Computer Science, New Jersey Institute of Technology, Newark, USA\\
\emails
xyf@seu.edu.cn,
pxu@njit.edu
}

\begin{document}

\maketitle
\begin{abstract}

Several scientific studies have reported the existence of the income gap among rideshare drivers based on demographic factors such as gender, age, race, etc. In this paper, we study the income inequality among rideshare drivers due to discriminative cancellations from riders, and the tradeoff between the income inequality (called fairness objective) with the system efficiency (called profit objective). We proposed an online bipartite-matching model where riders are assumed to arrive sequentially following a distribution known in advance. The highlight of our model is the concept of \emph{acceptance rate} between any pair of driver-rider types, where types are defined based on demographic factors. Specially, we assume each rider can accept or cancel the driver assigned to her, each occurs with a certain probability which reflects the acceptance degree from the rider type towards the driver type. We construct a bi-objective linear program as a valid benchmark and propose two LP-based parameterized online algorithms. 
Rigorous online competitive ratio analysis is offered to demonstrate the flexibility and efficiency of our online algorithms in balancing the two conflicting goals,
promotions of fairness and profit. Experimental results on a real-world dataset are provided as well, which confirm our theoretical predictions. 
  \end{abstract}

\section{Introduction}\label{sec:intro}

 Rideshares such as Uber and Lyft have received significant attention among research communities of computer science, operations research, and business, to name a few. One main research topic is the matching policy design of pairing drivers and riders, see, \eg
 \cite{danassis2019,curry2019mix,ashlagi2019edge,Patrick-18-JAI,BeiZ18,dickerson2018assigning,xu-aaai-19}. Most of the current work focuses on either the promotion of system efficiency or that of users' satisfaction or both.

In this paper, we study the fairness among rideshare drivers. There are several reports showing the earning gap among drivers based on their demographic factors such as age, gender and race, see, \eg~\cite{cook2018,rosenblat2016}. In particular, \cite{wage-gap} has reported that ``Black Uber and Lyft drivers earned \$13.96 an hour compared to the \$16.08 average for all other drivers'' and ``Women drivers reported earning an average of \$14.26 per hour, compared to \$16.61 for men''. The wage gap among drivers from different demographic groups is partially due to the discriminative cancellations from riders, which can be well spotted especially during off-peak hours when the number of riders is comparable or even less than that of drivers. Note that in rideshares like Uber and Lyft, after a driver accepts a rider: (1) all sensitive information of the driver such as name and photo will be accessible to the rider and (2) riders can cancel the driver for the first two minutes free of charge~\cite{cancel-policy}. This makes the discriminative cancellations from riders technically possible and economically worry-free. 

We aim to address the income disparity among drivers due to discriminative cancellations from riders and its tradeoff with system efficiency. Note that the two goals, promoting the group-level income equality among drivers and the system efficiency, are somewhat conflicting. Consider the off-peak hours for example, when riders are kinds of scarce resources. To maximize the system efficiency, rideshares like Uber should please riders by assigning them to their ``favorite'' drivers. This can effectively reduce any possible cancellations from riders and thus, minimize the risk of driving away riders to other rivals like Lyft. This measure, however, will offer those drivers ``popular'' among riders much more chances of getting orders than others and as a result, hurt the group-level income equality greatly.

In this paper, we propose two parameterized matching policies, which can smoothly tradeoff the above two goals with provable performances. We adopt the online-matching based model to capture the dynamics in rideshare, as commonly used before \cite{dickerson2018assigning,xu-aaai-19}. Assume a bipartite graph $G=(U,V,E)$ where $U$ and $V$ represent the sets of \emph{types} of offline drivers and online requests, respectively. Each driver type represents a specific demographic group (defined by gender, age, race, \etc) with a given location, while each request type represents a specific demographic group with a given starting and ending location. There is an edge $f=(u,v)$ if the driver (of type) $u$ is capable of serving the request (of type) $v$\footnote{For simplicity, we refer to a driver of type $u$ and a request of type $v$ directly as a driver $u$ and request $v$ when the context is clear.} (\eg the distance between them is below a given threshold). The online phase consists of $T$  rounds and in each round, a request $v \in V$ arrives dynamically. Upon its arrival an \emph{immediate and irrevocable} decision is required: either reject $v$ or assign it to a neighboring driver in $U$. We assume each $u$ has a matching capacity of $B_u \in \mathbb{Z}^{+}$, which captures the number of drivers belonging to the type $u$. Additionally, we have the following key assumptions in the model.

\xhdr{Arrivals of online requests}. We consider a finite time horizon $T$ (known to the algorithm). For each time-step or round $t \in [T] \doteq\{1,2,\ldots, T\}$, a request of type $v$ will be sampled (or $v$ arrives) from a known distribution $\{q_{v}\}$ such that $\sum_{v \in V} q_{v}=1$. Note that the sampling process is independent and identical across the online $T$ rounds. For each $v$, let $r_v=T \cdot q_v$, which is called the \emph{arrival rate} of request $v$ with $\sum_{v \in V} r_v=T$. Our arrival assumption is commonly called the \textit{known identical independent distributions} (KIID). This is mainly inspired from the fact that we can often learn the arrival distribution from historical logs~\cite{Yao2018deep,DBLP:conf/kdd/LiFWSYL18}. KIID is widely used in many practical applications of online matching markets including rideshare and crowdsourcing~\cite{xu-aaai-19,dickerson2018assigning,singer2013pricing,singla2013truthful}. 


\xhdr{Edge existence probabilities}. Each edge $f=(u,v)$ is associated with an existence probability $p_f \in (0,1]$, which captures the statistical acceptance rate of a request of type $v$ toward a driver of type $u$. The random process goes as follows. Once we assign $u$ to $v$, we observe an immediate random outcome of the existence, which is present (\ie $v$ accepts $u$) with probability $p_f$ and not ($v$ cancels $u$) otherwise.  We assume that (1) the randomness associated with the edge existence is independent across all edges; (2) the values $\{p_f\}$ are given as part of the input. The first assumption is motivated by individual choice and the second from the fact that historical logs can be used to compute such statistics with high precision.

\xhdr{Patience of requests}. Each request $v$ is associated with patience $\Del_v \in \mathbb{Z}^{+}$, which captures an upper bound of unsuccessful assignments the request $v$ can tolerate before leaving the platform. Under patience constraints, we can dispatch each request $v$ to at most $\Del_v$ different drives. Observe that we cannot broadcast $v$ to a set of at most $\Del_v$ different drives simultaneously. Instead, we should assign $v$ to at most $\Del_v$ distinct drives (maybe of the same type though) \emph{in a sequential manner} until either $v$ accepts one or $v$ leaves the system after running out of patience. We refer to this as the \emph{online probing process} (OPP). Note that OPP starts immediately after a request $v$ arrives if $v$ not rejected by the algorithm, and ends within one single round before the next request arrives.

We say an assignment $f=(u,v)$ is \emph{successful} if $u$ is assigned to $v$, and $v$~\emph{accepts} $u$ which occurs with probability $p_f$. Assume that the platform will gain a profit $w_f$ from a successful assignment $f=(u,v)$ (we call a match then). For a given policy $\ALG$, let $\cM$  be the set of (possibly random) successful assignments; we interchangeably use the term \emph{matching} to denote this set $\cM$. Inspired by the work of \cite{nanda2019,lesmana2019}, we define two objectives, namely \emph{profit} and  \emph{fairness}, which capture the system efficiency and group-level income equality among drivers, respectively. 

\begin{description}
\item [\textbf{Profit}:] The expected total profit over all matches obtained by the platform, which is defined as $ \E[\sum_{f \in \cM} w_f ]$.

\item [\textbf{Fairness}:] Let $\cM_u$ be the set of edges in $\cM$ incident to $u$. Define the fairness achieved by \ALG over all driver types as  $\min_{u \in U} \frac{\E[|\cM_u|]}{B_u}$.
\end{description}

\subsection{Preliminaries and Main Contributions}

\xhdr{Competitive ratio.} 
The competitive ratio is a commonly-used metric to evaluate the performance of online algorithms. 
Consider an online maximization problem for example. 
Let $\ALG(\mathcal{I})=\E_{I \sim \cI} [\ALG(I)]$ denote the expected performance of $\ALG$ on an input $\mathcal{I}$, where the expectation is taken over the random arrival sequence $I$.  
Let $\OPT(\mathcal{I})=\E[\OPT(I)]$ denote the expected \emph{offline optimal}, where $\OPT(I)$ refers to the optimal value after we observe the full arrival sequence $I$. 
Then, competitive ratio is defined as $\min_{\mathcal{I}} \frac{\ALG(\mathcal{I})}{\OPT(\mathcal{I})}$. 
It is a common technique to use an \LP to upper bound the $\OPT(\mathcal{I})$ (called the benchmark \LP) and hence get a valid lower bound on the target competitive ratio.  
In our paper, we conduct online competitive ratio analysis on both objectives.


\xhdr{Main contributions.} 
Our contributions can be summarized in the following three aspects. First, we propose a new online-matching based model to address the income inequality among drivers from different demographic groups and its trade-off with the system efficiency in rideshare. Second, we present a robust theoretical analysis for our model. We first construct a bi-objective linear program (\LP-\eqref{obj-1} and \LP-\eqref{obj-2}), which is proved to offer valid upper bounds for the respective maximum profit and fairness in the offline optimal. Then, we propose LP-based parameterized online algorithms \lpalg and \alg with provable performances on both objectives. We say an online algorithm achieves an $(\alp, \beta)$-competitive ratio if it achieves competitive ratios $\alp$ and $\beta$ on the profit and fairness against benchmarks \LP-\eqref{obj-1} and \LP-\eqref{obj-2}, respectively. Results in Theorems~\ref{thm:main-2} and~\ref{thm:hard} suggest that $\alg$ can achieve a nearly optimal ratio on each single objective either fairness or profit, though there is some space of improvement left for the summation of both ratios.

\begin{theorem}\label{thm:main-1}
$\lpalg(\alp, \beta)$ achieves a competitive ratio at least  $\Big(\alp \cdot \frac{1-1/e}{2}, \beta \cdot \frac{1-1/e}{2} \Big)$ simultaneously on the profit and fairness for any $\alp, \beta>0$ with $\alp+\beta \le 1$. 
\end{theorem}

\begin{restatable}{theorem}{mainTwo}\label{thm:main-2}
$\alg(\alp, \beta)$ achieves a competitive ratio at least  $\Big(\alp \cdot \frac{e-1}{e+1}, \beta \cdot \frac{e-1}{e+1} \Big) \sim (0.46 \cdot\alp, 0.46 \cdot\beta)$ simultaneously on the profit and fairness for any $\alp, \beta>0$ with $\alp+\beta \le 1$. 
\end{restatable}

\begin{theorem}\label{thm:hard}
No algorithm can achieve an $(\alp, \beta)$-competitive ratio simultaneously on the profit and fairness with $\alp+\beta>1$ or $\alp>0.51$ or $\beta>0.51$ using \LP-\eqref{obj-1} and \LP-\eqref{obj-2} as benchmarks.
\end{theorem}
 
Last, we test our model and algorithms on a real dataset collected from a large on-demand taxi dispatching platform. Experimental results confirm our theoretical predictions and demonstrate the flexibility of our algorithms in tradeoffing the two conflicting objectives and their efficiency compared to natural heuristics.

\section{Related Work}

Fairness in operations is an interesting topic which has a large body of work~\cite{Bertsimas2011ThePO,Bertsimas2012OnTE,Chen2018WhyAF,Lyu2019MultiObjectiveOR,Cohen2019PriceDW,Ma2020GrouplevelFM,Chen2020SameDayDW}.
Here is a few recent work addressing the fairness issue in rideshares.~\cite{suhr2019} proposed two notions of amortized fairness for fair distribution of income among rideshare drivers, one is related to absolute income equality, while the other is averaged income equality over active time.~\cite{lesmana2019} considered nearly the same two objectives as proposed in this paper. Note that both of the aforementioned work considered  an essential offline setting in the way that all arrivals of online requests are known in advance by considering a short time window. Additionally, both ignore the potential cancellations from riders, and assume each rider will accept the assigned driver surely (\ie all $p_f=1$). The recent work~\cite{nanda2019} studied an interesting ``dual'' setting to us. They focused on the peak hours and examined the fairness on the rider side due to discriminative cancellations from drivers. 

Our model technically belongs to a more general optimization paradigm, called \emph{Multi-Objective Optimization}. Here are a few theoretical work which studied the design of approximation or online algorithms to achieve a bi-criterion approximation and/or online competitive ratios, see,~\eg~\cite{ravi1993many,grandoni2009,korula2013bicriteria,aggarwal2014,esfandiari2016bi}. The work of \cite{bansal2012lp,BSSX17,fata2019multi} have the closest setting to us: each edge has an independent existence probability and each vertex from the offline and/or online side has a  patience constraint on it. However, all investigated one single objective: maximization of the total profit over all matched edges.

\section{Valid Benchmarks for Profit and Fairness}

We first present our benchmark LPs and then an LP-based parameterized algorithm.
For each edge $f=(u,v)$, let $x_f$ be the expected number of probes on edge $f$ (\ie assignments of $v$ to $u$ but not necessarily matches) in the offline optimal. For each $u$ ($v$), let $E_u$ ($E_v$) be the set of neighboring edges incident to $u$ ($v$). Consider the following bi-objective LP.

\begin{alignat}{2}
\textstyle \max & \textstyle ~~\sum_f w_f x_f p_f  && \label{obj-1} \\
\textstyle \max \min_{u \in U} & \textstyle ~~\frac{\sum_{f \in E_u} x_f p_f}{B_u} &&  \label{obj-2} \\
\text{s.t.} & \textstyle \sum_{f \in E_u} x_f p_f \le B_u  &&~~ \textstyle \forall u \in U \label{cons:match-u} \\ 
 & \textstyle \sum_{f \in E_v} x_f \le \Delta_v \cdot r_v  && ~~ \textstyle \forall v \in V \label{cons:pat-v} \\ 
  & \textstyle \sum_{f \in E_v} x_f p_f \le r_v  && ~~ \textstyle \forall v \in V \label{cons:mat-v}  \\
 & \textstyle 0 \le x_f \le r_v  && ~~ \textstyle \forall f \in E_v\label{cons:edge}
 \end{alignat}

Let \LP-\eqref{obj-1} and \LP-\eqref{obj-2} denote the two LPs with the respective objectives \eqref{obj-1} and   \eqref{obj-2}, each with Constraints \eqref{cons:match-u}, \eqref{cons:pat-v}, \eqref{cons:mat-v}, \eqref{cons:edge}. Note that we can rewrite Objective \eqref{obj-2} as a linear one like $\max \eta$ with additional linear constraints as $\eta \le \frac{\sum_{f \in E_u} x_f p_f}{B_u} $ for all $u \in U$. For presentation convenience, we keep the current compact version. The validity of \LP-\eqref{obj-1} and \LP-\eqref{obj-2} as benchmarks for our two objectives can be seen in the following lemma.

\begin{lemma} \label{lem:LP}
 \LP-\eqref{obj-1} and \LP-\eqref{obj-2} are valid benchmarks for the two respective objectives, profit and fairness. In other words, the optimal values to \LP-\eqref{obj-1} and \LP-\eqref{obj-2} are valid upper bounds for the expected profit and fairness achieved by the offline optimal, respectively.
\end{lemma}

\begin{proof}
We can verify that objective functions~\eqref{obj-1} and~\eqref{obj-2} each captures the exact expected profit and fairness achieved by the offline optimal by the linearity of expectation. To prove the validity of the benchmark for each objective, it suffices to show the feasibility of all constraints for any given offline optimal. Recall that for each edge $f$, $x_f$ denotes the expected number of probes on $f$ (\ie assignments of $u$ to $v$ but not necessarily matches) in the offline optimal. Constraint \eqref{cons:match-u} is valid since each driver $u$ has a matching capacity of $B_u$. Note that the expected arrivals of $v$ during the whole online phase is $r_v$ and $v$ can be probed at most $\Del_v$ times upon each online arrival. Thus, the expected number of total probes and matches over all edges incident to $v$ should be no more than $r_v \Del_v$ and $r_v$, respectively. This rationalizes Constraints \eqref{cons:pat-v} and \eqref{cons:mat-v}. The last constraint is valid, since for each edge, the expected number of probes should be no more than that of arrivals. Therefore, we justify the feasibility of all constraints for any given offline optimal. \end{proof}

\section{LP-based Parameterized Algorithms}

The following lemma suggests that for any online algorithm \ALG, the worst-case scenario (\ie the instance on which \ALG achieves the lowest competitive ratio) arrives when each driver type has a unit matching capacity. 

\begin{restatable}{lemma}{lemunit}\label{lem:unit}
Let $\ALG$ be an online algorithm achieving an $(\alp, \beta)$-competitive ratio on instances with unit matching capacity (\ie all $B_u$=1). We can twist $\ALG$ to $\ALG'$ such that $\ALG'$ achieves at least  an $(\alp, \beta)$-competitive ratio on instances with general integral matching capacities. 
\end{restatable}

\begin{proof}
Let $\ALG$ be online algorithm with an $(\alp, \beta)$-competitive ratio when all $u$ have a unit matching capacity. Consider a given instance $I'$ with general matching capacities. We can create a corresponding instance $I$ with unit capacities by replacing each $u$ with a set $S_u$ of identical copies of $u$ where $|S_u|=B_u$. Note that for any feasible solution $\{x'_f\}$ to  \LP-\eqref{obj-1} and \LP-\eqref{obj-2} on the instance $I'$, we can create another feasible solution $\{x_f\}$ to  \LP-\eqref{obj-1} and \LP-\eqref{obj-2} on $I$ where $x_f =x'_f/B_u$ for every $f \in S_u$ and $u$. We can verify that the objective values of \LP-\eqref{obj-1} and \LP-\eqref{obj-2} each remains the same on $\{x'_f\}$ and $\{x_f\}$. Similarly, let $\{x_f\}$ be a given feasible solution of \LP-\eqref{obj-1} and \LP-\eqref{obj-2} on the instance $I$, we can create a feasible solution $\{x'_f\}$ on $I'$ where $x'_f= \sum_{f \in S_u}x_f$ for each $u$. We can verify that the objective value of \LP-\eqref{obj-1} remains the same on $\{x_f\}$ and $\{x'_f\}$. Let $\{y'_f\}$ be an optimal solution \LP-\eqref{obj-2} on the instance $I$ and w.lo.g. assume that all $y'_w\doteq \sum_{f \in E_w} y'_f p_f$ are the same over all $w \in \cup_{u} S_u$ (otherwise we can decrease some $\{y'_f\}$ until all $y'_w$ are the same).  
Thus, we claim that the objective value of \LP-\eqref{obj-2} remains the same on $\{y'_f\}$ and $\{y_f\}$. From the above analysis, we conclude that the optimal LP values of  \LP-\eqref{obj-1} and \LP-\eqref{obj-2} each remains the same on $I'$ and $I$.

Now let $\ALG'$ be such an online algorithm on $I'$ that it first replaces $I'$ with $I$ and then apply $\ALG$ to $I$. We can verify that $\ALG'$ is valid online algorithm on $I'$ since (1) each $u$ will be matched at most $B_u$ times since each $w \in S_u$ will be matched at most once when applying $\ALG$ to $I$; and (2) each $v$ will be probed at most $\Del_v$ times upon arrival. Let $P(I)$ and $F(I)$ be the profit and fairness achieved by $\ALG$ on $I$. Similarly let $P(I')$ and $F(I')$ be the profit and fairness achieved by $\ALG'$ on $I'$. Let $W=\cup_u S_u$ be the set of offline vertices in $I$, and for each $w \in W$, let $\gam_w$ be the expected number of matches of $w$ when $\ALG$ is applied to $I$. Observe that $P(I')=P(I)$ and 
\[F(I')=\min_u \frac{\sum_{w \in S_u} \gam_w }{B_u}\ge \min_{w \in \cup_u S_u} \frac{\gam_w}{1}=F(I).\]

Note that benchmark LP values of  \LP-\eqref{obj-1} and \LP-\eqref{obj-2} each remains the same on $I$ and $I'$. Thus, we get our claim. 
\end{proof}

From Lemma~\ref{lem:unit}, we assume \emph{unit capacity for all driver types} throughout this paper w.l.o.g. In the following, we will present a warm-up algorithm ($\lpalg$) and then another refined algorithm (\alg), which can be viewed as a polished version of $\lpalg$ with simulation-based attenuation techniques. The main idea of $\alg$ is primarily inspired by the work \cite{BSSX17}. Both $\lpalg$ and $\alg$ invoke the following dependent rounding techniques (denoted by GKPS) introduced by \cite{gandhi2006dependent}. For simplicity, we state a simplified version of GKPS tailored to star graphs which suffices in our paper.

Recall that $E_v$ is the set of edges incident to $v$ in the compatible graph $G$. GKPS is such a dependent rounding technique that takes as input a fractional vector $\z=\{z_f, f \in E_v, z_f \in [0,1]\}$ on $E_v$,  and output a random binary vector $\Z=\{Z_f, f \in E_v\}$, which satisfied the following properties. (1) \textbf{Marginal distribution}: $\E[Z_f]=z_f$ for all $f \in E_v$; (2) \textbf{Degree preservation}: $\Pr[\sum_{f \in E_v} Z_f \le \sum_{f \in E_v} z_f]=1$; (3) \textbf{Negative correlation}: For any pair of edges  $f, f' \in E_v$, $\E[Z_f=1|Z_{f'}=1] \le z_f$.

Throughout this section, we assume (1) $\x^*=\{x^{*}_{f}\}$ and $\y^*=\{y^*_{f}\}$ are optimal solutions to \LP-\eqref{obj-1} and \LP-\eqref{obj-2} respectively; (2) $(\alp, \beta)$ are two given parameters with $0\le \alp, \beta \le 1, \alp+\beta \le 1$; (3) $B_u=1$ for all $u$ from Lemma~\ref{lem:unit}; (4) $\x^v=\{ x^*_f/r_v, f \in E_v\}$ and $\y^v=\{ y^*_f/r_v, f\in E_v\}$, which are scaled solutions from $\x^*$ and $\y^*$ respectively restricted on $E_v$. Note that from Constraints~\eqref{cons:pat-v}, \eqref{cons:mat-v} and \eqref{cons:edge}, we have that $\x^v$ and $\y^v$ are two fractional solutions on $E_v$ and each has a total sum at most $\Del_v$.

\xhdr{The first algorithm $\lpalg(\alp,\beta)$}. Let an online vertex $v$ arrive at $t$. Our job is to probe at most $\Del_v$ edges in $E_{v}$ until $v$ is matched. Let $\z$ be a given fractional solution on $E_v$. $\lpalg(\alp,\beta)$ invokes the following procedures (denoted by $\sr(\z)$) as a subroutine during each online round: it first selects a set $\cS_v$ of at most $\Del_v$ edges from $E_v$ in a random way guided by a given fractional vector $\z$ on $E_v$ and then follows a random order to process all edges in $\cS_v$ one by one. The details of $\sr$ are stated in Algorithm~\ref{alg:sr}.
\vspace{-3mm}
\begin{algorithm}[ht!]
\DontPrintSemicolon
Apply GKPS to the fractional vector $\z$ and let $\Z$ be the random binary vector output. \;
Choose a random permutation $\pi$ over $E_v$. \;
 Follow the order $\pi$ to process each $f=(u,v) \in E_v$ until $v$ is matched: \;
 \Indp \If { $Z_f=1$ and $u$ is available}{ Probe the edge $f$ (\ie assign $v$ to $u$).}
 \Else{Skip to the next one.}
\caption{Sub-Routine $\sr(\z)$: Dependent rounding combined with random permutation}
\label{alg:sr}
\end{algorithm}
\vspace{-2mm}

Based on  $\sr$, the main idea of $\lpalg(\alp,\beta)$ is as simple as follows: each round when an online vertex $v$ arrives, it invokes $\sr({\x}^v)$ and $\sr({\y}^v)$ with probabilities $\alp$ and $\beta$ respectively. Recall that ${\x}^v$ and ${\y}^v$ are the scaled optimal solutions to \LP-\eqref{obj-1} and \LP-\eqref{obj-2} restricted to $E_v$, each has a total sum at most $\Del_v$. Thus, when we run $\sr({\x}^v)$ or $\sr({\y}^v)$ after $v$ arrives online, we will probe at most $\Del_v$ edges incident to $v$ since the final rounded binary vector has at most $\Del_v$ ones due to Property of Degree Preservation in the dependent rounding.   The details of $\lpalg(\alp,\beta)$ are as follows.
\vspace{-3mm}
\begin{algorithm}[h!]
\DontPrintSemicolon
Let $v$ arrive at time $t$. \;
With probability $\alp$, run  $\sr({\x}^v)$. \;
With probability $\beta$, run  $\sr({\y}^v)$. \;
 With probability $1-\alp-\beta$, reject $v$.
\caption{An LP-based warm-up algorithm: $\lpalg(\alp, \beta)$}
\label{alg:lp-alg}
\end{algorithm}
\vspace{-2mm}

We conduct an edge-by-edge analysis. It would suffice to show that each $f$ is probed with probability at least $x_f^* \cdot \alp \cdot (1-1/e)/2$ and $y_f^* \cdot \beta \cdot  (1-1/e)/2$ in $\lpalg(\alp, \beta)$. Then by linearity of expectation, we can get Theorem \ref{thm:main-1}. Focus  on a given $u$ and a time $t \in [T]$. Let $\SF_{u,t}$ be the event that $u$ is \emph{available} at (the beginning of) $t$.

\begin{lemma}\label{lem:1}
For any given $u$ and $t \in [T]$, we have $\textstyle \Pr[\SF_{u,t}] \ge \textstyle \Big(1-\frac{1}{T}\Big)^{t-1}$.
\end{lemma}

\begin{proof}
Recall that we assume w.l.o.g. that each $B_u$=1 due to Lemma~\ref{lem:unit}. For each given $\ell<t$ and $f=(u,v) \in E_u$, let $X_{f, \ell}$ indicate if $v$ arrives at time $t$; $Y_{f,\ell}$ indicate if $f$ is probed during round $\ell$; $Z_{f,\ell}$ indicate if $f$ is present when probed. Note that in each subroutine of $\sr({\x}^v)$ and $\sr({\y}^v)$ after $v$ arrives, $f$ will be probed only when the final rounded vector has the entry one on $f$. Thus we claim that $\E[Y_{f,\ell}] \le  \alp x_f^*/r_v+\beta y_f^*/r_v$ due to Property of Marginal Distribution in dependent rounding and statements of $\lpalg(\alp, \beta)$. Thus,
\vspace{-3mm}
 \begin{align*}
&\textstyle \small  \Pr[\SF_{u,t}]= \textstyle \prod_{\ell<t} \Pr\Big[\sum_{f \in E_u} X_{f,\ell} Y_{f,\ell} Z_{f,\ell}=0\Big]\\
&\textstyle = \textstyle \small \prod_{\ell<t} \Big( 1-\Pr\Big[ \sum_{f \in E_u}X_{f,\ell} Y_{f,\ell} Z_{f,\ell} \ge 1 \Big]\Big) \\
 &\textstyle = \textstyle \prod_{\ell<t} \Big(  1-\sum_{f \in E_u} \frac{r_v}{T} \Big(\alp \frac{x_f^*}{r_v}+\beta \frac{y_f^*}{r_v} \Big) p_f \Big)\\
 &\textstyle = \textstyle \prod_{\ell<t} \Big(  1-\frac{1}{T}\sum_{f \in E_u}  \Big(\alp x_f^* p_f+\beta y_f^* p_f \Big)\Big)  \\
 &\textstyle \ge  \Big(1-\frac{1}{T}\Big)^{t-1}
 \end{align*}
\vspace{-3mm}
 \end{proof}

Now assume $\SF_{u,t}$ occurs (\ie $u$ is available at $t$). Consider a given $f=(u,v)$ and let $\bo_{f,t}$ indicate $f$ is probed during round $t$ in $\lpalg(\alp, \beta)$. Notice that $\bo_{f,t}$ occurs if (1) $v$ arrives at time $t$ and (2) $f$ is probed either in $\sr({\x}^v)$ or $\sr({\y}^v)$.

\begin{lemma}\label{lem:2}
$\Pr[\bo_{f,t} |\SF_{u,t}] \ge \frac{\alp x_f^*}{2T}, \Pr[\bo_{f,t} |\SF_{u,t}] \ge \frac{\beta y_f^*}{2T}$.
\end{lemma}
\begin{proof}
We focus on the first inequality and try to show that $f$ is probed at $t$ in $\sr({\x}^v)$ with probability at least  $\frac{\alp x_f^*}{2T}$ (including the probability of its online arrival).  Observe that events $v$ arrives at time $t$ and $\lpalg(\alp, \beta)$ runs the subroutine $\x^v$ both happen with probability $\frac{\alp r_v}{T}$. Let $\X^v$ be the rounded binary vector from ${\x}^v$ and we use $X^v_f$ to denote its entry on $f$. Let $E_{v,\neg f}$ be the set of edges in $E_v$ excluding $f=(u,v)$. For each $f' \in E_{v,\neg f}$, let $Y_{f'}$ indicate if $f'$ falls before $f$ in the random order $\pi$ and $Z_{f'}$ indicate if $f'$ is present when probed. Thus we have
\begin{align}
&\textstyle \small \Pr[\bo_{f,t}|\SF_{u,t}] \\
&\textstyle \small  \ge \frac{\alp r_v}{T} \Pr[X^v_f=1] \Pr\Big[\sum_{f'  \in E_{v,\neg f}} X^v_{f'} Y_{f'}Z_{f'} =0 | X^v_f=1\Big]  \label{ineq:wu-1}\\
&\textstyle = \frac{\alp r_v}{T}\frac{x_f^*}{r_v} \Big(1-\Pr\Big[\sum_{f'  \in E_{v, \neg f}} X^v_{f'} Y_{f'}Z_{f'} \ge 1 | X^v_f=1\Big] \Big) \label{ineq:wu-2} \\
&\textstyle \ge \frac{\alp x_f^*}{T} \Big(1-\E\Big[\sum_{f'  \in E_{v, \neg f}} X^v_{f'} Y_{f'}Z_{f'} | X^v_f=1\Big] \Big)  \label{ineq:wu-3} \\
&\textstyle \ge \frac{\alp x_f^*}{T}  \Big(1-\sum_{f'  \in E_{v, \neg f}} \E\Big[X^v_{f'} Y_{f'}Z_{f'}  | X^v_f=1\Big] \Big)  \label{ineq:wu-4} \\
&\textstyle \ge \frac{\alp x_f^*}{T}  \Big(1-\sum_{f'  \in E_{v, \neg f}}  \frac{x^*_{f'}}{r_v} \frac{p_f}{2} \Big) \label{ineq:wu-5} \\
&\textstyle \ge \frac{\alp x_f^*}{T} \frac{1}{2}. \label{ineq:wu-6}
\end{align}
Inequality~\eqref{ineq:wu-3} follows from Markov's inequality.  Inequality~\eqref{ineq:wu-5} is due to these two observations: (1) $\E[X^v_{f'}| X_f^v=1] \le x_f^*/r_v$ due to negative correlation in dependent rounding and (2)~$\E[Y_{f'}]=1/2$, $\E[Z_{f'}]=p_f$. Inequality~\eqref{ineq:wu-6} follows from  the fact $\sum_{f' \in E_v} x^*_{f'} p_{f'} \le r_v$ due to Constraint~\eqref{cons:mat-v}. Following a similar analysis, we can prove the second part.
\end{proof}

Now we have all ingredients to prove the main Theorem~\ref{thm:main-1}.

\begin{proof}
Consider a given $f=(u,v) \in E$, let $\kap_f^{\x}$ and $\kap_f^{\y}$ be the expected number of \emph{successful} probes of $f$ in $\sr(\x^v)$ and $\sr(\y^v)$ respectively. Here a probe of $f=(v,u)$ is successful iff $u$ is available when we assign $v$ to $u$ (but no necessarily means $f$ is present). 
\begin{align*}  
 \textstyle \kappa_f^{\x}  &\ge \textstyle \sum_{t=1}^T  \Pr[\SF_{u,t}] \Pr[\bo_{f,t}|\SF_{u,t}] \\
 & \textstyle \ge    \sum_{t=1}^T  \Big(1-\frac{1}{T}\Big)^{t-1} \frac{\alp x_f^*}{2T} \sim \frac{\alp x^*_f (1-1/e)}{2}
   \end{align*}
 The last term is obtained after taking $T \rightarrow \infty$. Similarly, we can show that $\kappa_f^\y \ge \frac{\beta y^*_f (1-1/e)}{2}$. 
 
Let $\pro(\alp, \beta)$ be the expected total profit obtained by $\lpalg(\alp, \beta)$. By linearity of expectation, we have $\pro(\alp, \beta) \ge \frac{(1-1/e) \alp}{2} \sum_{f \in E} x^*_f p_e w_e $. From Lemma~\ref{lem:LP}, we know that the expected profit in offline optimal is upper bounded by $\sum_{f \in E} x_f^* p_e w_e$. Thus we claim that $\lpalg(\alp, \beta)$ achieves a ratio at least $\alp(1-1/e)/2$ on the profit. Similarly, we can argue that $\lpalg(\alp, \beta)$ achieves a ratio at least $\beta (1-1/e)/2$ on the fairness.  
\end{proof}

\xhdr{The second algorithm $\alg(\alp,\beta)$}. Inspired by~\cite{BSSX17}, we can improve at least the theoretical performance of $\lpalg$ with attenuation techniques applied to edges and (offline) vertices. The motivation behind is very simple. Note that edges in $E_v$ are competing for each other since we have to stop probing whenever $v$ is matched. Thus, attenuating those edges which win the higher chance of probing over others can potentially boost the worst-case performance.  

Let $\{\gam_t, \mu_t |t \in [T]\}$ be such a series that is defined as 
$\gam_1=1, \mu_t=1-\gam_t/2, \gam_{t+1}=\gam_t(1-\mu_t/T)$.

Let $E_{v,t}$ be the set of available edges $f=(u,v) \in E_v$ at time $t$ (\ie $u$ is available at $t$). The formal description of $\alg$  is stated in Algorithm~\ref{alg:att}. We defer the proofs of Theorems~\ref{thm:main-2} to the Appendix.

\begin{algorithm}[ht!]
\DontPrintSemicolon
\For{$t=1,2,\ldots, T$}{ 
Apply vertex-attenuation such that each $u \in U$ is available at $t$ with probability equal to $\gam_t$. \label{alg:att-3}\;
Let $v$ arrive at time $t$. \;
With probability $\alp$, \;
\Indp Run  $\sr({\x}^v)$. Apply edge-attenuation such that each edge $f\in E_{v,t}$ is probed in $\sr({\x}^v)$ with probability equal to $\mu_t x_f^*/r_v$.\label{alg:att-1}\; 
\Indm With probability $\beta$, \;
\Indp Run  $\sr({\y}^v)$. Apply edge-attenuation such that each edge $f\in E_{v,t}$ is probed in $\sr({\y}^v)$ with probability equal to $\mu_t y_f^*/r_v$. \label{alg:att-2}\;
\Indm  With probability $1-\alp-\beta$, reject $v$.}
\caption{An LP-based algorithm after attenuation: $\alg(\alp, \beta)$}
\label{alg:att}
\end{algorithm}

\section{Hardness Results}

We prove Theorem~\ref{thm:hard} in this section. Consider the below example.  

\begin{example}\label{exam:hard}
Consider a graph which consists of $n$ identical units, each unit $i \in [n]$ is a star graph which includes the center of $v_i$ and two other neighbors $u_i^a$ and $u_i^b$. Set $p_{i,a}=1$ and $p_{i,b}=\ep$ where we use $\{i,a\}$ ($\{i,b\}$) to index the edges $(u_i^a,v_i)$ and $(u_i^b, v_i)$ respectively. Assume that (1) unit edge weight on all edges; (2) $T=n$ and unit arrival rate on all $v_i$ (\ie all $r_v=1$); (3) unit matching capacity on all $u$ (\ie all $B_u$=1); and (4) unit patience on all $v$ (\ie all $\Del_v=1$).  

Let $\OPTP$ and $\OPTF$ be the optimal LP values of \LP-\eqref{obj-1} and \LP-\eqref{obj-2} on the above example respectively. We can verify that: (1) $\OPTP=n$, where there is a unique optimal solution $x^*_{i,a}=1$ and $x^*_{i,b}=0$ for all $i \in [n]$; (2) $\OPTF=\ep/(1+\ep)$, where there is a unique optimal solution $y^*_{i,a}=\frac{\ep}{1+\ep}$ and $y^*_{i,b}=\frac{1}{1+\ep}$ for all $i \in [n]$. 
\end{example}
Now based on Example~\ref{exam:hard}, we prove the below lemma.

\begin{lemma}\label{lem:hard-cr}
Consider Example~\ref{exam:hard} and assume  \LP-\eqref{obj-1} and \LP-\eqref{obj-2} as benchmarks. We have (1) no algorithm can achieve a competitive ratio larger than $1-1/e$ on the profit; (2) no algorithm can achieve competitive ratios on the profit and fairness with a sum larger than $1$. 
\end{lemma}

\begin{proof}
Consider a given online algorithm  $\ALG$, in which the expected number of probes for $(u_i^a, v_i)$ and $(u_i^b, v_i)$ are $\alp_i$ and $\beta_i$ for each $i \in [n]$, respectively. Let $\ALGP$ and $\ALGF$ be the profit and fairness achieved by $\ALG$. We have that 
$\ALGP=\sum_{i \in [n]} (\alp_i+\beta_i \ep), \ALGP=\min_{i \in [n]} \Big( \alp_i, \beta_i \ep \Big)$. Set $\alp\doteq \sum_{i \in [n]} \alp_i$ and $\beta \doteq \sum_{i \in [n]} \beta_i$. Note that (1) $\alp+\beta \le n$, and (2) $\alp \le (1-1/e)n$. The latter inequality is due to each $\alp_i \le 1-1/e$. Thus, the sum of competitive ratios on profit and fairness should be 
\begin{align*}
&\small \frac{\ALGP}{\OPTP}+\frac{\ALGF}{\OPTF}=\frac{\sum_{i \in [n]} \alp_i+\beta_i \ep }{n}+\frac{\min_{i \in [n]} \Big( \alp_i, \beta_i \ep \Big) }{\ep/(1+\ep)} \\
&\small \le \frac{\alp+\ep \beta}{n}+\frac{\beta (1+\ep)}{n}=\frac{\alp+\beta+2\ep \beta}{n} \le 1+2 \ep.
\end{align*}
As for profit, we see that $\frac{\ALGP}{\OPTP}=\frac{\alp+\ep \beta}{n} \le 1-1/e+\ep$.
\end{proof}
 
Based on the example presented in Lemma 5 of Section 3.1 of~\cite{fata2019multi}, we can get a stronger version of statement (2) in Lemma~\ref{lem:hard-cr}, which states that no online algorithm can get an online ratio better than $0.51$ for either the profit or fairness based on  \LP-\eqref{obj-1} and \LP-\eqref{obj-2}. Summarizing all analysis we prove Theorem~\ref{thm:hard}.

\begin{figure*}[h!]
  \centering
  \subfigure[$B=10$]{
    \label{fig:pfva_10}
    \includegraphics[width=0.22\textwidth]{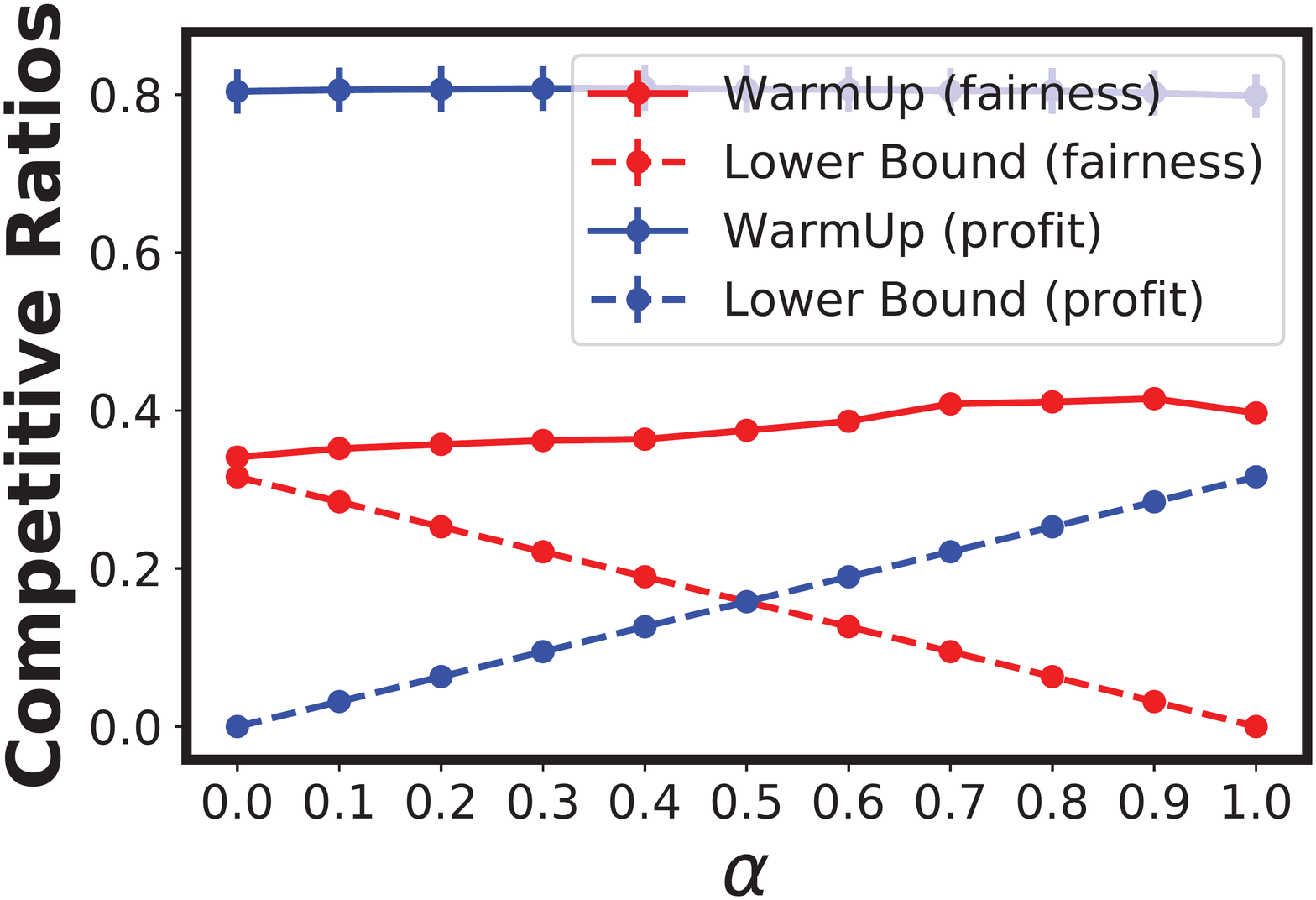}
  }
  \subfigure[$B=15$]{
    \label{fig:pfva_15}
    \includegraphics[width=0.22\textwidth]{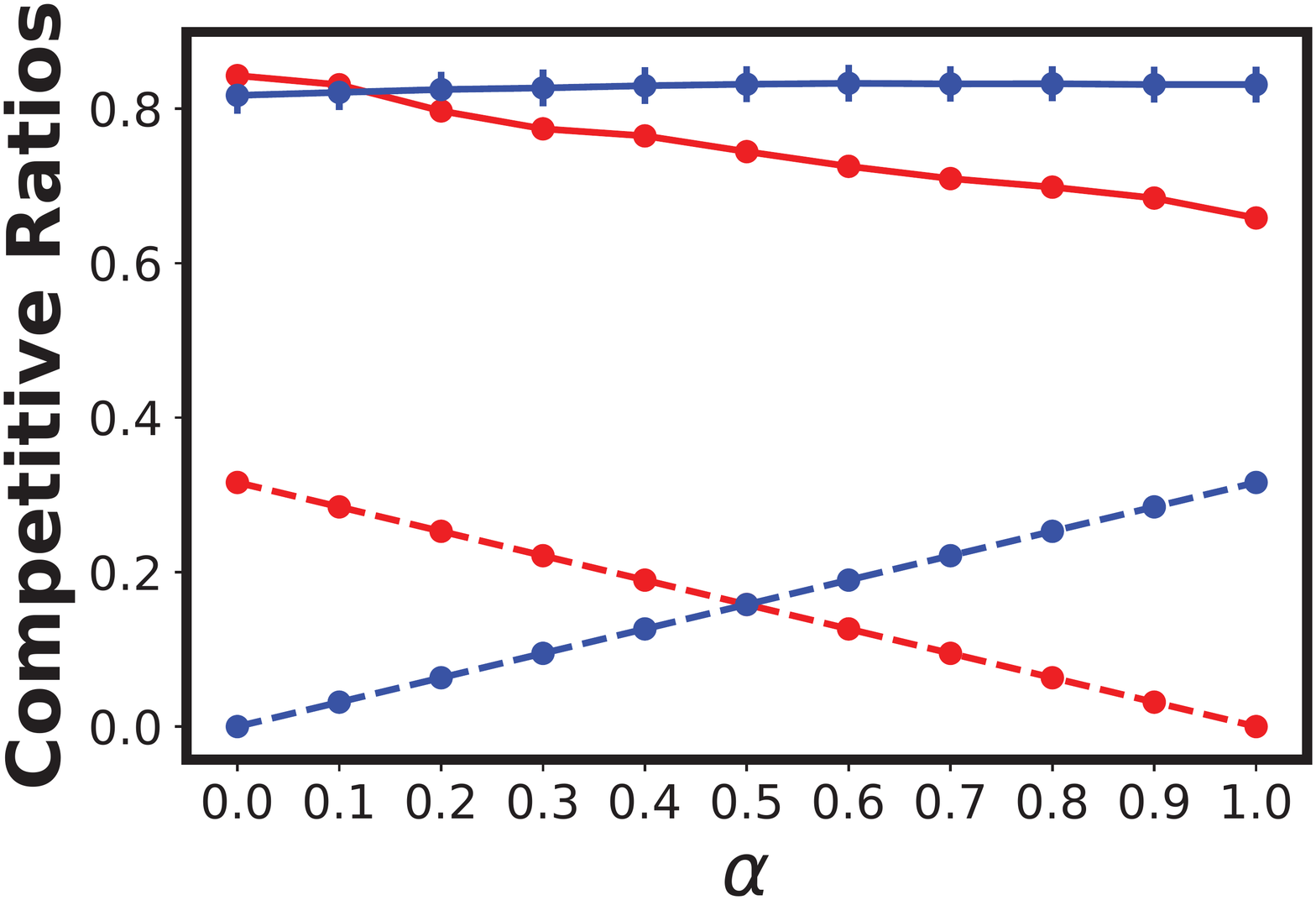}
  }
  \subfigure[$B=20$]{
    \label{fig:pfva_20}
    \includegraphics[width=0.22\textwidth]{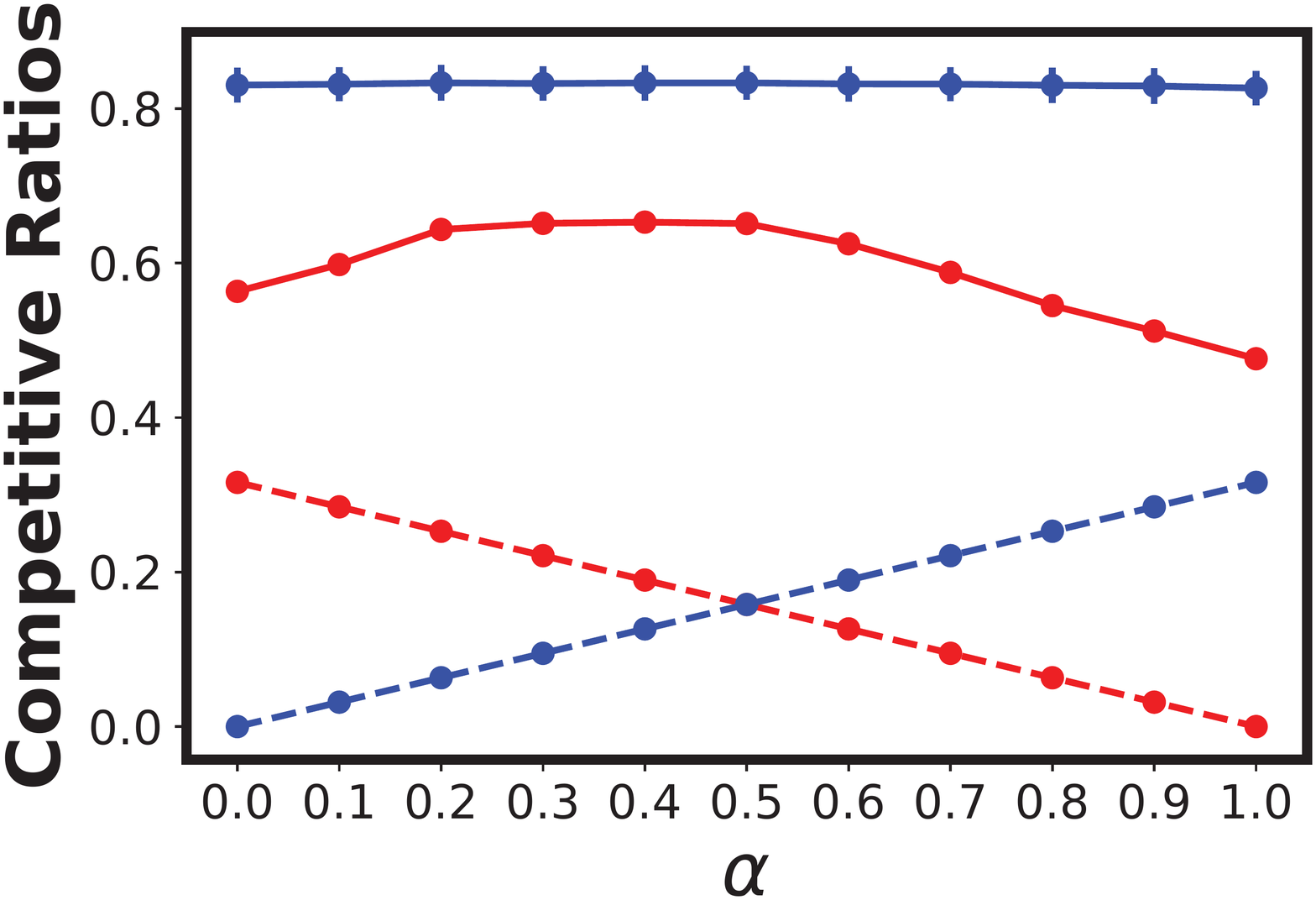}
  }
  \subfigure[$B=25$]{
    \label{fig:pfva_25}
    \includegraphics[width=0.22\textwidth]{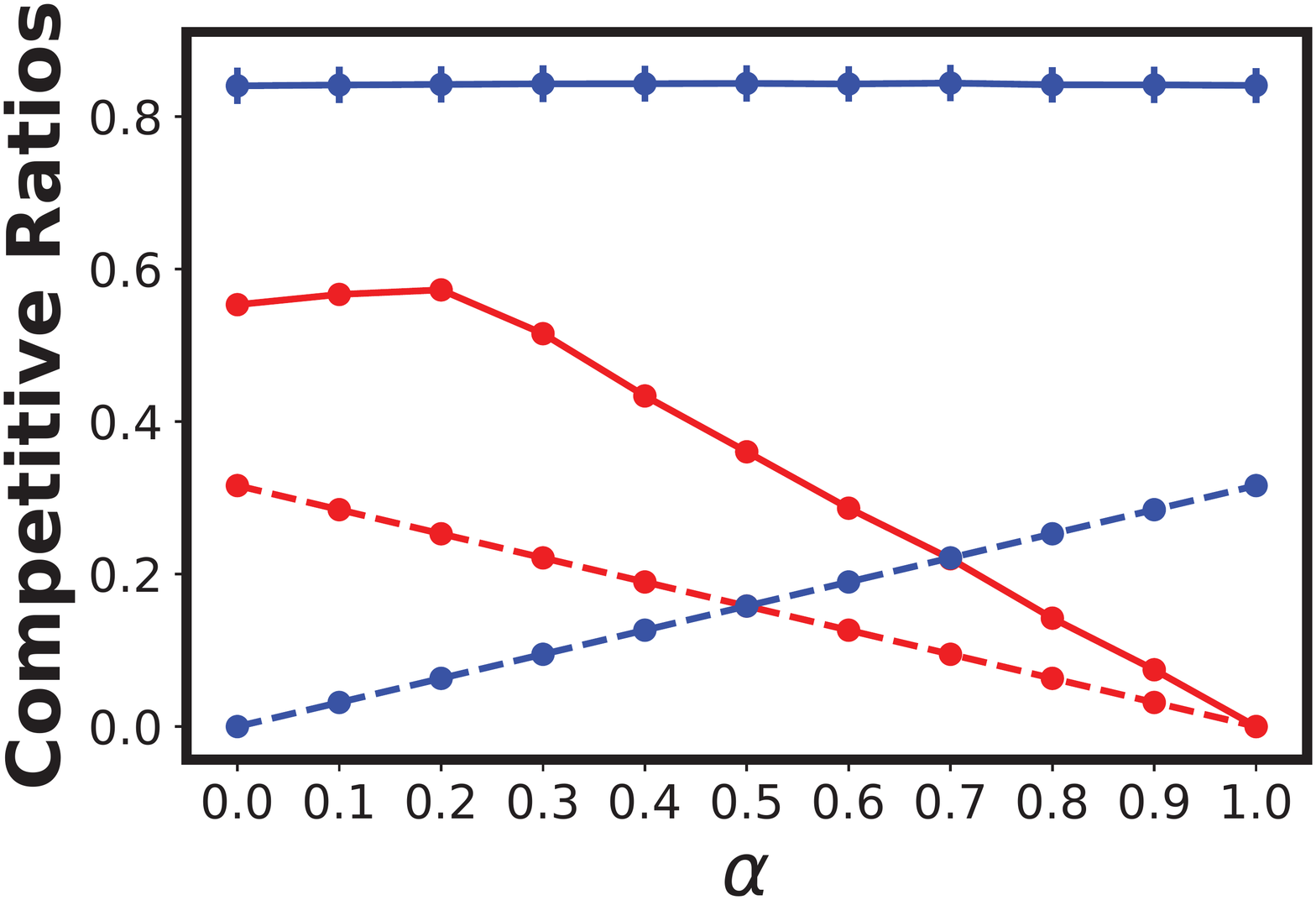}
  }
  \caption{Competitive ratios for profit and fairness with different values of $\alpha$ and $\beta$ with $\alpha+\beta=1$.}
  \label{fig:pfva_u}
\end{figure*}

\begin{figure*}[h!]
  \centering
  \vspace{-3mm}
  \subfigure[$B=10$]{
    \label{fig:fps_10}
    \includegraphics[width=0.22\textwidth]{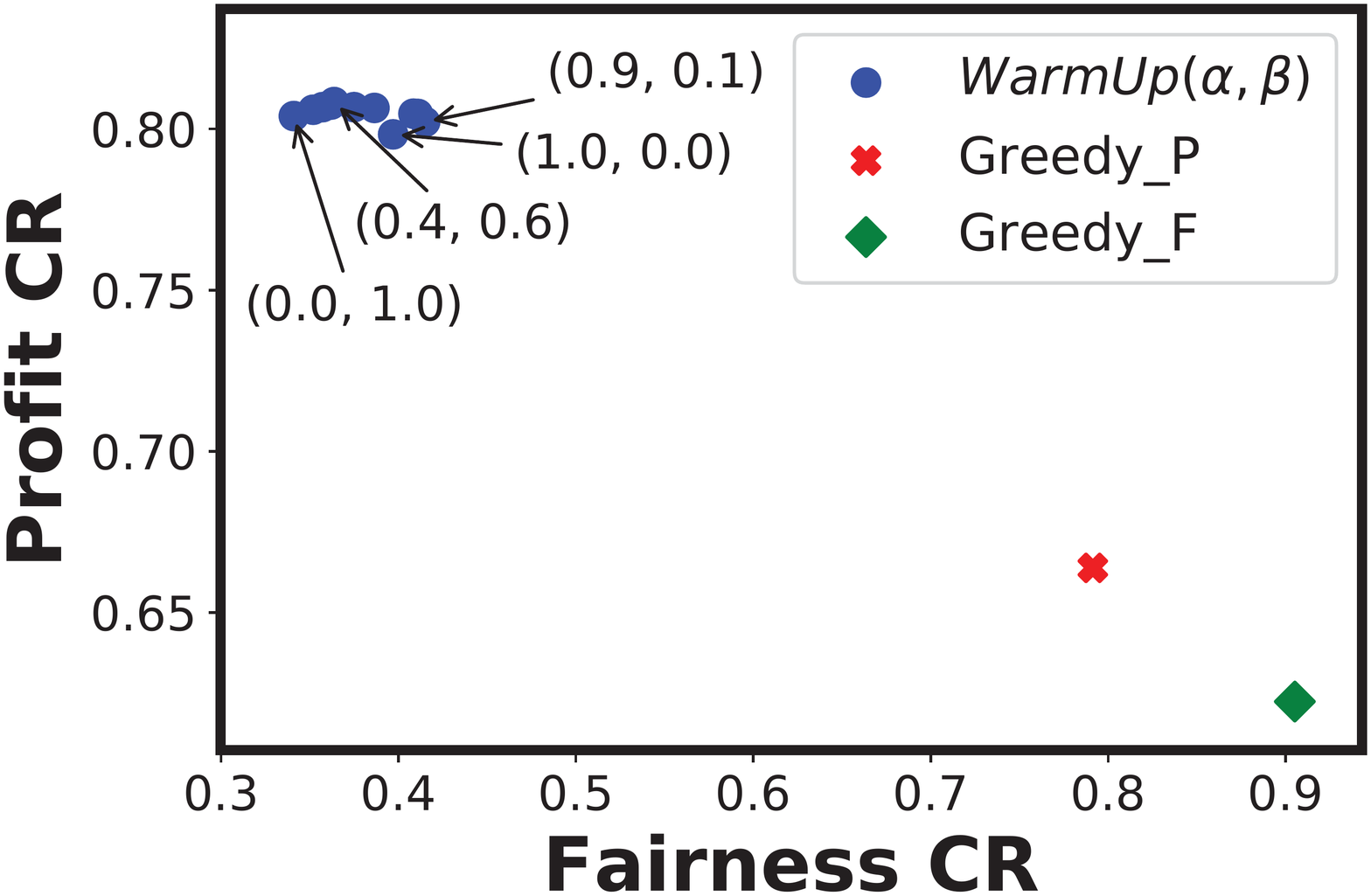}
  }
  \subfigure[$B=15$]{
    \label{fig:fps_15}
    \includegraphics[width=0.22\textwidth]{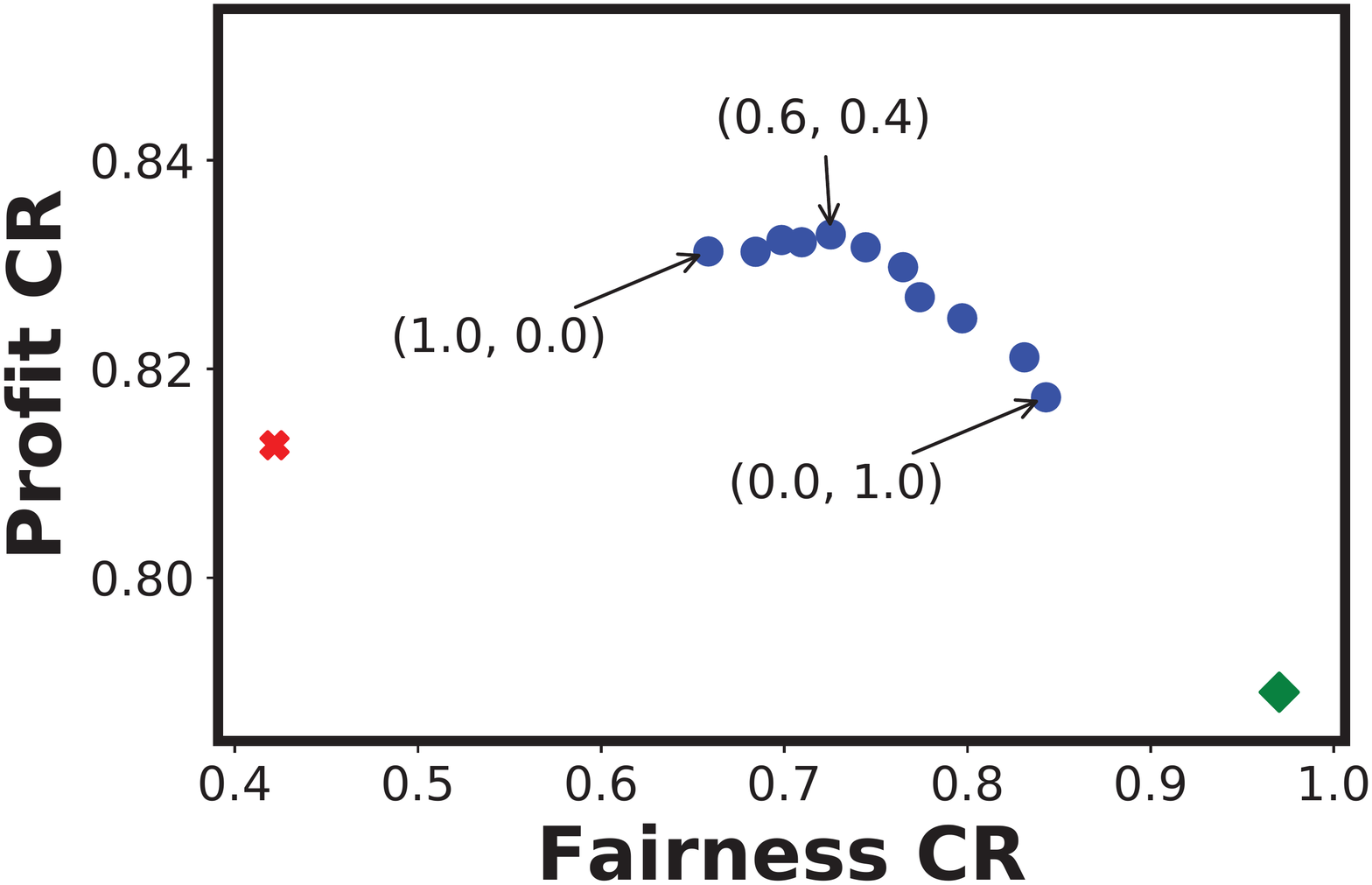}
  }
  \subfigure[$B=20$]{
    \label{fig:fps_20}
    \includegraphics[width=0.22\textwidth]{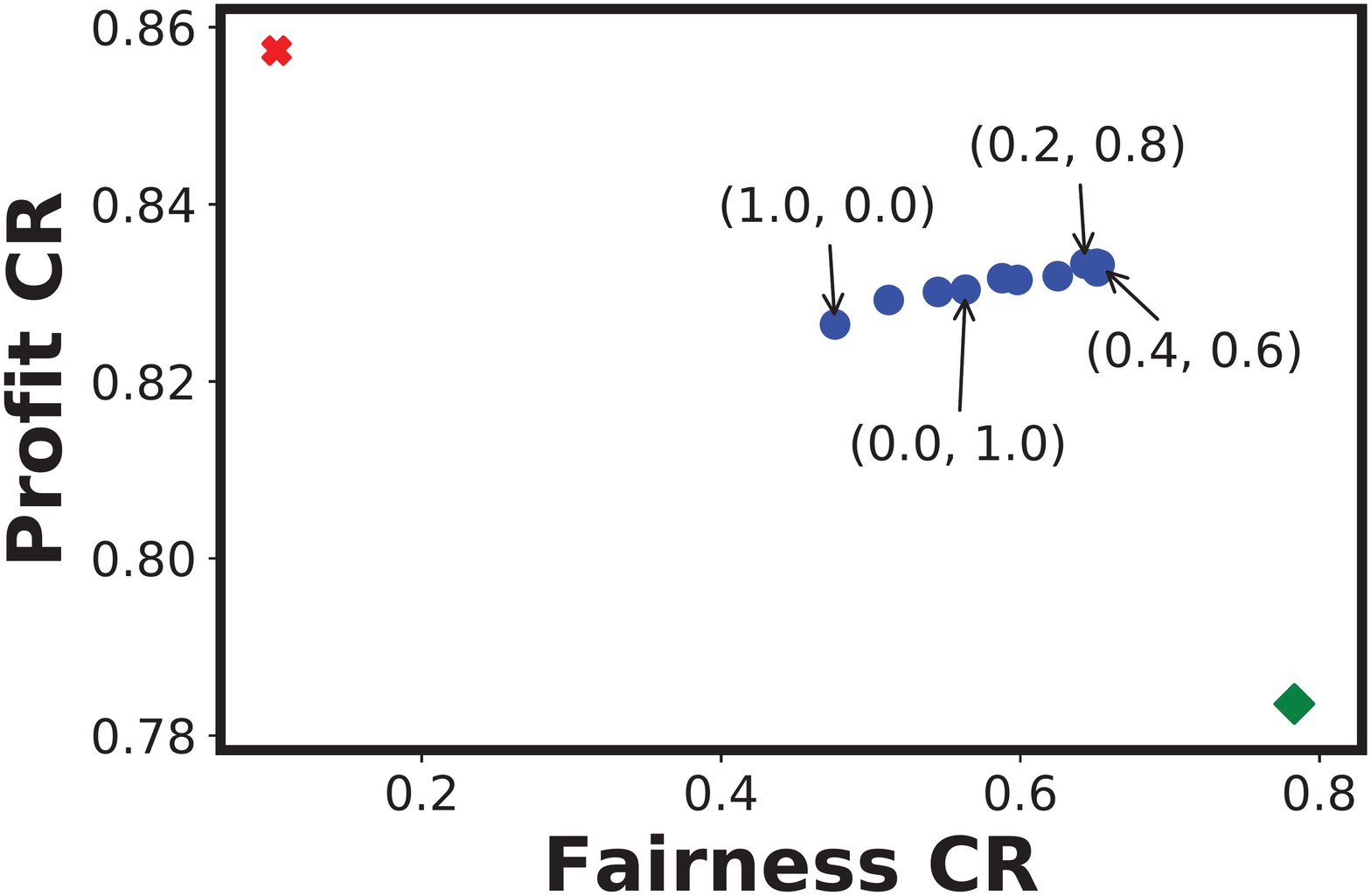}
  }
  \subfigure[$B=25$]{
    \label{fig:fps_25}
    \includegraphics[width=0.22\textwidth]{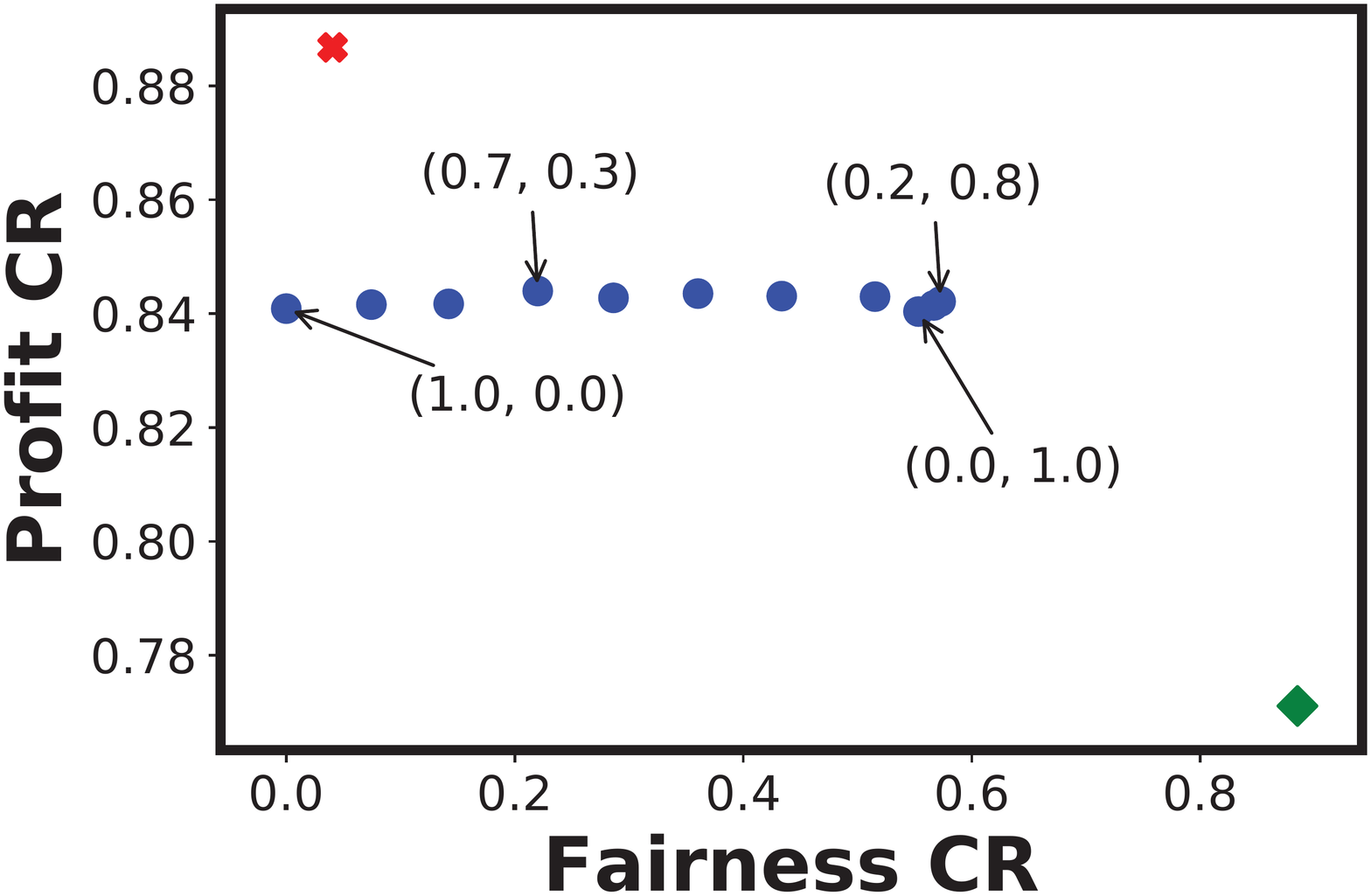}
  }
  \caption{Performance comparisons with Greedy\_P and Greedy\_F.}
  \label{fig:fps_u}
\end{figure*}

\section{Experiments}
In this section, we describe our experimental results on a real dataset: the New York City yellow cabs dataset\footnote{\url{http://www.andresmh.com/nyctaxitrips/}} which contains the trip histories for thousands of taxis across Manhattan, Brooklyn, and Queens.

\xhdr{Data preprocessing}.
The dataset is collected during the year of 2013.
Each trip record includes the (desensitized) driver's license, the pick-up and drop-off locations for the passenger, the duration and distance to complete the trip, the starting and ending time of the trip and some other information such as the number of customers. Although the demographics of the drivers and riders are not recorded in the original dataset, we synthesize the racial demographics for riders and drivers in a similar way to~\cite{nanda2019}. To simplify the demonstration, we consider a single demographic factor of the race only, which takes two possible options between ``disadvantaged'' (D) or ``advantaged'' (A). We set the ratio of D to A to be $1:2$ among riders, which roughly matches the racial demographics of NYC~\cite{ridersref}. Similarly, we set the ratio of D to A among drivers to be $1:2$~\cite{driversref}. The acceptance rates among the four possible  driver-rider pairs (based on race status only), (A,A), (A,D), (D,A), (D,D), are set to be $0.6,0.1,0.1$ and $0.3$, respectively. These probabilities are then scaled up by a factor $\eta$ such that $p_f = \eta+(1 - \eta)\cdot p_f$. In our experiments we set $\eta = 0.5$. Note that we can apply our model straightforwardly to the case when the real-world distribution of $\{p_f\}$ values is known or can be learned.  We collect records during the off-peak period of 4--5 PM when a lot of drivers are on the road while the requests are relatively lower than peak hours.
On January $31$, 2013, $20,701$ trips were completed in the off-peak hour (from 16:00 to 17:00), compared to $35,109$ trips in the peak hour (from 19:00 to 20:00). We focus on longitude and latitude ranging from $(-73,-75)$ and $(40.4,40.95)$ respectively.
We partition the area into $40\times11$ grids with equal size.
Each grid is indexed by a unique number to represent a specific pick-up and drop-off location. 

We construct the compatibility  graph $G=(U,V,E)$ as follows. Each $u \in U$ represents a driver type which has attributes of the starting location and race. Each $v \in V$ represents a request type which has attributes of the starting location, ending location, and race. We downsample from all driver and request types such that $|U|=57$ and $|V|=134$.  For each driver type $u$, we assign its capacity $B_u$ with a random value uniformly sampled from $[1,B]$ where we vary $B \in \{10,15,20,25\}$.  
For each request of type $v$, we sample a random patience value $\Del_v$ uniformly from $\{1,2\}$ and a random arrival rate $r_v \sim \mathcal{N}(5, 1)$  (Normal distribution), and then set $T=\sum_{v \in V} r_v$. We add an edge $f=(u,v)$ if the Manhattan distance between starting location of request type $v$ and the location of driver type $u$ is not larger than $1$. The profit $w_f$ for each $f$ is defined as the normalized trip length of the request type $v$ such that $0 \le w_f \le 1$. 


\xhdr{Algorithms}.
We test the $\lpalg(\alp,\beta)$ with $\alp+\beta=1$ against two natural heuristic baselines, namely Greedy-P (short for Greedy-Profit) and Greedy-F (short for Greedy-Fairness)\footnote{A future direction is to consider a hybrid version of Greedy-P and Greedy-F, which will optimize the two objectives simultaneously.}. Suppose a request type of $v$ arrives at time $t$. Recall that $E_v$ is the set of neighboring edges incident to $v$ (\ie the set of assignments feasible to $v$). Let $E'_v \subseteq E_v$ be the set of \emph{available}  assignments $f=(u,v)$ such that there exists at least one drive of type $u$ at $t$. For Greedy-P, it will repeat greedily selecting an available assignment $f \in E'_v$ with the maximum weight $w_fp_f$ over $E'_v$ (breaking ties arbitrarily) until either $v$ accepts a driver or $v$ runs out of patience. In contrast, Greedy-F will repeat greedily selecting an available $f=(u^*,v) \in E'_v$ with $u^*$ having the least matching rate before either $v$ accepts a driver or leaves the system. We run all $\lpalg(\alp,\beta)$ algorithms for $1000$ independent trials and take the average as the expectations.
We also run Greedy-P and Greedy-F for $1000$ instances and take the average values as the final performance. Note that we use \LP-\eqref{obj-1} and \LP-\eqref{obj-2} as the default benchmarks for profit and fairness, respectively.

\xhdr{Results and discussions}.
Figure~\ref{fig:pfva_u} shows the results of competitive ratios for the proposed algorithm with different values of $\alp$ with ($\beta=1-\alp$). We can observe that the profit and fairness competitive ratios of $\lpalg$ always stay above the theoretical lower bounds (in dotted lines), as predicted in Theorem~\ref{thm:main-1}. The gaps between performances and lower bounds suggest that theoretical worst scenarios occur rarely in the real world. Note that when $B =25$ and $\alp=1$ as shown in Figure~\ref{fig:pfva_25}, the lower bound is tight and matches the fairness performance. 

Figure~\ref{fig:fps_u} shows the profit and fairness performances of $\lpalg$  compared to Greedy-P and Greedy-F. Here are a few interesting observations. (1) As for profit, Greedy-P can always beat Greedy-F but not necessarily for $\lpalg$. The advantage of Greedy-P over $\lpalg$ becomes more apparent when $B$ is large and less when $B$ is small. Note that in our experiment, the expected total number of arrivals of riders is fixed and therefore, $B$ directly controls the degree of imbalance between drivers and riders.  When $B$ is larger, we have more available drivers compared to riders and thus, Greedy-P will outperform all the rest for profit. When $B$ is small, however, we really need to carefully design the policy to boost profit. That's why $\lpalg$ becomes dominant. (2) As for fairness, Greedy-F seemingly can always dominate the rest, though $\lpalg$ shows high flexibility in the fairness performance. $\lpalg$ shows a relatively low sensitivity toward the first parameter $\alp$ for profit while high sensitivity toward the second parameter $\beta$ for fairness: the latter becomes particularly obvious when $B$ is large.

\section{Conclusion}

In this paper, we present a flexible approach for matching requests to drivers to balance the two conflicting goals, maximizations of income equality among all rideshare drivers and the total revenue earned by the system. Our proposed approach allows the policy designer to specify how fair and how profitable they want the system to be via two separate parameters.  Extensive experimental results on the real-world dataset show that our proposed approaches not only are far above the theoretical lower bounds but also can smoothly tradeoff the two objectives between the two natural heuristics. Our work opens a few directions for future research. The most direct one is to shorten the gap between the sum of ratios of profit and fairness achieved by $\alg$ (which is $0.46$). It will be interesting to give a tighter online analysis than what are presented here or offer a sharper hardness result which suggests the sum of the two ratios should be much lower than $1$.

{\small
\bibliographystyle{named}
\bibliography{stable_ref}

\begin{thebibliography}{}

\bibitem[\protect\citeauthoryear{Aggarwal \bgroup \em et al.\egroup
  }{2014}]{aggarwal2014}
Gagan Aggarwal, Yang Cai, Aranyak Mehta, and George Pierrakos.
\newblock Biobjective online bipartite matching.
\newblock In {\em International Conference on Web and Internet Economics},
  pages 218--231. Springer, 2014.

\bibitem[\protect\citeauthoryear{Ashlagi \bgroup \em et al.\egroup
  }{2019}]{ashlagi2019edge}
Itai Ashlagi, Maximilien Burq, Chinmoy Dutta, Patrick Jaillet, Chris Sholley,
  and Amin Saberi.
\newblock Edge weighted online windowed matching.
\newblock In {\em Proceedings of the Nineteenth ACM Conference on Economics and
  Computation}, 2019.

\bibitem[\protect\citeauthoryear{Bansal \bgroup \em et al.\egroup
  }{2012}]{bansal2012lp}
Nikhil Bansal, Anupam Gupta, Jian Li, Juli{\'a}n Mestre, Viswanath Nagarajan,
  and Atri Rudra.
\newblock When lp is the cure for your matching woes: Improved bounds for
  stochastic matchings.
\newblock {\em Algorithmica}, 63(4):733--762, 2012.

\bibitem[\protect\citeauthoryear{Bei and Zhang}{2018}]{BeiZ18}
Xiaohui Bei and Shengyu Zhang.
\newblock Algorithms for trip-vehicle assignment in ride-sharing.
\newblock AAAI '18, pages 3--9, 2018.

\bibitem[\protect\citeauthoryear{Bertsimas \bgroup \em et al.\egroup
  }{2011}]{Bertsimas2011ThePO}
D.~Bertsimas, V.~Farias, and N.~Trichakis.
\newblock The price of fairness.
\newblock {\em Oper. Res.}, 59:17--31, 2011.

\bibitem[\protect\citeauthoryear{Bertsimas \bgroup \em et al.\egroup
  }{2012}]{Bertsimas2012OnTE}
D.~Bertsimas, V.~Farias, and N.~Trichakis.
\newblock On the efficiency-fairness trade-off.
\newblock {\em Manag. Sci.}, 58:2234--2250, 2012.

\bibitem[\protect\citeauthoryear{Brubach \bgroup \em et al.\egroup
  }{2018}]{BSSX17}
Brian Brubach, Karthik Sankararaman, Aravind Srinivasan, and Pan Xu.
\newblock Attenuate locally, win globally: Attenuation-based frameworks for
  online stochastic matching with timeouts.
\newblock {\em Algorithmica}, 04 2018.

\bibitem[\protect\citeauthoryear{Chen and Wang}{2018}]{Chen2018WhyAF}
Y.~Chen and H.~Wang.
\newblock Why are fairness concerns so important? lessons from a shared
  last-mile transportation system.
\newblock {\em Civil \& Environmental Engineering eJournal}, 2018.

\bibitem[\protect\citeauthoryear{Chen \bgroup \em et al.\egroup
  }{2020}]{Chen2020SameDayDW}
X.~Chen, Tong Wang, B.~Thomas, and Marlin~W. Ulmer.
\newblock Same-day delivery with fairness.
\newblock {\em ArXiv}, abs/2007.09541, 2020.

\bibitem[\protect\citeauthoryear{Cohen \bgroup \em et al.\egroup
  }{2019}]{Cohen2019PriceDW}
Maxime~C. Cohen, Adam~N. Elmachtoub, and Xiao Lei.
\newblock Price discrimination with fairness constraints.
\newblock 2019.

\bibitem[\protect\citeauthoryear{Cook \bgroup \em et al.\egroup
  }{2018}]{cook2018}
Cody Cook, Rebecca Diamond, Jonathan Hall, John~A List, and Paul Oyer.
\newblock The gender earnings gap in the gig economy: Evidence from over a
  million rideshare drivers.
\newblock Technical report, National Bureau of Economic Research, 2018.

\bibitem[\protect\citeauthoryear{Curry \bgroup \em et al.\egroup
  }{2019}]{curry2019mix}
Michael Curry, John~P Dickerson, Karthik~Abinav Sankararaman, Aravind
  Srinivasan, Yuhao Wan, and Pan Xu.
\newblock Mix and match: Markov chains and mixing times for matching in
  rideshare.
\newblock In {\em International Conference on Web and Internet Economics},
  pages 129--141. Springer, 2019.

\bibitem[\protect\citeauthoryear{Danassis \bgroup \em et al.\egroup
  }{2019}]{danassis2019}
Panayiotis Danassis, Marija Sakota, Aris Filos-Ratsikas, and Boi Faltings.
\newblock Putting ridesharing to the test: Efficient and scalable solutions and
  the power of dynamic vehicle relocation.
\newblock {\em arXiv preprint arXiv:1912.08066}, 2019.

\bibitem[\protect\citeauthoryear{Dickerson \bgroup \em et al.\egroup
  }{2018}]{dickerson2018assigning}
John~P. Dickerson, Karthik~Abinav Sankararaman, Aravind Srinivasan, and Pan Xu.
\newblock Assigning tasks to workers based on historical data: Online task
  assignment with two-sided arrivals.
\newblock In {\em Proceedings of the 17th International Conference on
  Autonomous Agents and MultiAgent Systems}, AAMAS '18, pages 318--326, 2018.

\bibitem[\protect\citeauthoryear{Dough}{2019}]{cancel-policy}
Dough.
\newblock Uber fees: How to avoid the cancellation fee, cleaning fee, and more.
\newblock
  \url{https://www.ridesharingdriver.com/uber-fees-cancellation-booking-cleaning-fees/},
  2019.
\newblock Accessed: 2019-12-27.

\bibitem[\protect\citeauthoryear{Esfandiari \bgroup \em et al.\egroup
  }{2016}]{esfandiari2016bi}
Hossein Esfandiari, Nitish Korula, and Vahab Mirrokni.
\newblock Bi-objective online matching and submodular allocations.
\newblock In {\em Advances in Neural Information Processing Systems}, pages
  2739--2747, 2016.

\bibitem[\protect\citeauthoryear{Fata \bgroup \em et al.\egroup
  }{2019}]{fata2019multi}
Elaheh Fata, Will Ma, and David Simchi-Levi.
\newblock Multi-stage and multi-customer assortment optimization with inventory
  constraints.
\newblock {\em Available at SSRN 3443109}, 2019.

\bibitem[\protect\citeauthoryear{Gandhi \bgroup \em et al.\egroup
  }{2006}]{gandhi2006dependent}
Rajiv Gandhi, Samir Khuller, Srinivasan Parthasarathy, and Aravind Srinivasan.
\newblock Dependent rounding and its applications to approximation algorithms.
\newblock {\em Journal of the ACM (JACM)}, 53(3):324--360, 2006.

\bibitem[\protect\citeauthoryear{Grandoni \bgroup \em et al.\egroup
  }{2009}]{grandoni2009}
Fabrizio Grandoni, Ramamoorthi Ravi, and Mohit Singh.
\newblock Iterative rounding for multi-objective optimization problems.
\newblock In {\em European Symposium on Algorithms}, pages 95--106. Springer,
  2009.

\bibitem[\protect\citeauthoryear{Hinchliffe}{2017}]{wage-gap}
Emma Hinchliffe.
\newblock Yes, there's a wage gap for uber and lyft drivers based on age,
  gender and race.
\newblock \url{https://mashable.com/2017/01/18/uber-lyft-wage-gap-rideshare/},
  2017.
\newblock Accessed: 2019-12-27.

\bibitem[\protect\citeauthoryear{Korula \bgroup \em et al.\egroup
  }{2013}]{korula2013bicriteria}
Nitish Korula, Vahab~S Mirrokni, and Morteza Zadimoghaddam.
\newblock Bicriteria online matching: Maximizing weight and cardinality.
\newblock In {\em International conference on web and internet economics},
  pages 305--318. Springer, 2013.

\bibitem[\protect\citeauthoryear{Lesmana \bgroup \em et al.\egroup
  }{2019}]{lesmana2019}
Nixie~S Lesmana, Xuan Zhang, and Xiaohui Bei.
\newblock Balancing efficiency and fairness in on-demand ridesourcing.
\newblock In {\em Advances in Neural Information Processing Systems}, pages
  5310--5320, 2019.

\bibitem[\protect\citeauthoryear{Li \bgroup \em et al.\egroup
  }{2018}]{DBLP:conf/kdd/LiFWSYL18}
Yaguang Li, Kun Fu, Zheng Wang, Cyrus Shahabi, Jieping Ye, and Yan Liu.
\newblock Multi-task representation learning for travel time estimation.
\newblock KDD '18, pages 1695--1704, 2018.

\bibitem[\protect\citeauthoryear{Lowalekar \bgroup \em et al.\egroup
  }{2018}]{Patrick-18-JAI}
Meghna Lowalekar, Pradeep Varakantham, and Patrick Jaillet.
\newblock Online spatio-temporal matching in stochastic and dynamic domains.
\newblock {\em Artificial Intelligence}, 261:71 -- 112, 2018.

\bibitem[\protect\citeauthoryear{Lyu \bgroup \em et al.\egroup
  }{2019}]{Lyu2019MultiObjectiveOR}
Guodong Lyu, Wang~Chi Cheung, C.~Teo, and H.~Wang.
\newblock Multi-objective online ride-matching.
\newblock 2019.

\bibitem[\protect\citeauthoryear{Ma and Xu}{2020}]{Ma2020GrouplevelFM}
Will Ma and Pan Xu.
\newblock Group-level fairness maximization in online bipartite matching.
\newblock {\em ArXiv}, abs/2011.13908, 2020.

\bibitem[\protect\citeauthoryear{Nanda \bgroup \em et al.\egroup
  }{2019}]{nanda2019}
Vedant Nanda, Pan Xu, Karthik~Abinav Sankararaman, John~P Dickerson, and
  Aravind Srinivasan.
\newblock Balancing the tradeoff between profit and fairness in rideshare
  platforms during high-demand hours.
\newblock {\em arXiv preprint arXiv:1912.08388}, 2019.

\bibitem[\protect\citeauthoryear{Ravi \bgroup \em et al.\egroup
  }{1993}]{ravi1993many}
R~Ravi, Madhav~V Marathe, SS~Ravi, Daniel~J Rosenkrantz, and Harry~B Hunt~III.
\newblock Many birds with one stone: Multi-objective approximation algorithms.
\newblock In {\em Proceedings of the 25nd Annual {ACM} Symposium on Theory of
  Computing}, STOC '93, pages 438--447. Citeseer, 1993.

\bibitem[\protect\citeauthoryear{Review}{2019}]{ridersref}
World~Population Review.
\newblock New york city population.
\newblock
  \url{http://worldpopulationreview.com/us-cities/new-york-city-population/},
  2019.
\newblock Accessed: 2020-01-12.

\bibitem[\protect\citeauthoryear{Rosenblat \bgroup \em et al.\egroup
  }{2016}]{rosenblat2016}
Alex Rosenblat, Karen Levy, Solon Barocas, and Tim Hwang.
\newblock Discriminating tastes: Customer ratings as vehicles for bias.
\newblock {\em Available at SSRN 2858946}, 2016.

\bibitem[\protect\citeauthoryear{Singer and Mittal}{2013}]{singer2013pricing}
Yaron Singer and Manas Mittal.
\newblock Pricing mechanisms for crowdsourcing markets.
\newblock WWW '13, pages 1157--1166, 2013.

\bibitem[\protect\citeauthoryear{Singla and Krause}{2013}]{singla2013truthful}
Adish Singla and Andreas Krause.
\newblock Truthful incentives in crowdsourcing tasks using regret minimization
  mechanisms.
\newblock WWW '13, pages 1167--1178, 2013.

\bibitem[\protect\citeauthoryear{S\"{u}hr \bgroup \em et al.\egroup
  }{2019}]{suhr2019}
Tom S\"{u}hr, Asia~J. Biega, Meike Zehlike, Krishna~P. Gummadi, and Abhijnan
  Chakraborty.
\newblock Two-sided fairness for repeated matchings in two-sided markets: A
  case study of a ride-hailing platform.
\newblock KDD '19, pages 3082--3092, New York, NY, USA, 2019. ACM.

\bibitem[\protect\citeauthoryear{Thorton and Kiersten}{2017}]{driversref}
S.~Thorton and M.~Kiersten.
\newblock Please provide a demographic breakdown of uber's drivers.
\newblock
  \url{https://askwonder.com/research/please-provide-demographic-breakdown-uber-s-drivers-aoi5j3r00},
  2017.
\newblock Accessed: 2020-01-12.

\bibitem[\protect\citeauthoryear{Yao \bgroup \em et al.\egroup
  }{2018}]{Yao2018deep}
Huaxiu Yao, Fei Wu, Jintao Ke, Xianfeng Tang, Yitian Jia, Siyu Lu, Pinghua
  Gong, Jieping Ye, and Zhenhui Li.
\newblock Deep multi-view spatial-temporal network for taxi demand prediction.
\newblock AAAI '18, pages 2588--2595, 2018.

\bibitem[\protect\citeauthoryear{Zhao \bgroup \em et al.\egroup
  }{2019}]{xu-aaai-19}
Boming Zhao, Pan Xu, Yexuan Shi, Yongxin Tong, Zimu Zhou, and Yuxiang Zeng.
\newblock Preference-aware task assignment in on-demand taxi dispatching: An
  online stable matching approach.
\newblock AAAI '19, 2019.

\end{thebibliography}
}

\onecolumn

\section{Appendix}

\subsection{Proof of Theorem \ref{thm:main-2}}\label{sec:app-main-2}

\mainTwo*

Here are two main ingredients to prove Theorem \ref{thm:main-2}. Assume at time $t$, each $u$ is available with probability $\gam_t$, and $v$ arrives at $t$. Recall that $E_{v,t}$ is the set of available edges in $E_v$ at time $t$.

\begin{lemma}\label{lem:app-1}
Each edge $f \in E_{v,t}$ will be probed with probability at least $ \mu_t x_f^*/r_v$ if $\sr(\x^v)$ is invoked and at least  $ \mu_t y_f^*/r_v$ if $\sr(\y^v)$ is invoked without attenuation. 
\end{lemma}
\begin{proof}
From definition, we see $\mu_t=1-\gam_t/2$. We present a similar but refined version of proof to that in Lemma~\ref{lem:2}. Consider a given $f=(u,v) \in E_{v,t}$, and let $\kap^{\x}_{f,t}$ be the probability that $f$ is probed in $\sr(\x^v)$. Similar to before, let $\X^v$ be the rounded binary vector from ${\x}^v$ and $X^v_f$  be the entry of $\X^v$ on $f$. 
Let $E_{v,\neg f}$ be the set of edges in $E_v$ excluding $f=(u,v)$. For each $f' \in E_{v,\neg f}$, let $Y_{f'}$ indicate if $f'$ falls before $f$ in the random order $\pi$ and $Z_{f'}$ indicate if $f'$ is present when probed. We introduce an additional indicator $H_{f'}$ to show if $f'=(u',v) \in E_{v,t}$ (\ie $u'$ is available at $t$). Note that here we assume not only $u$ is available (\ie $H_f=1$) but also $v$ arrives at $t$ and $\sr(\x^v)$ is invoked. Thus, we have
\begin{align}
\kap^{\x}_{f,t} & \ge \Pr[ X^v_f=1]\Pr \Big[ \sum_{f'\in E_{v,t}} X^v_{f} Y_{f'} Z_{f'}=0  ~ |~ H_f=1, X^v_f=1 \Big] \\
&=\frac{x_f^*}{r_v} \Pr \Big[ \sum_{f'\in E_{v,\neg f}} H_{f'} X^v_{f'} Y_{f'} Z_{f'}=0  ~ |~ H_f=1, X^v_f=1 \Big] \\
& =\frac{x_f^*}{r_v} \Big(1-\Pr \Big[ \sum_{f'\in E_{v,\neg f}} H_{f'} X^v_{f'} Y_{f'} Z_{f'} \ge 1  ~ |~ H_f=1, X^v_f=1 \Big] \Big) \\
& \ge \frac{x_f^*}{r_v} \Big(1-\E \Big[ \sum_{f'\in E_{v,\neg f}} H_{f'} X^v_{f'} Y_{f'} Z_{f'}~|~ H_f=1, X^v_f=1 \Big] \Big) \\
& = \frac{x_f^*}{r_v} \Big(1-\sum_{f'\in E_{v,\neg f}} \E \Big[  H_{f'} X^v_{f'} Y_{f'} Z_{f'}~|~ H_f=1, X^v_f=1 \Big] \Big) \\
& \ge  \frac{x_f^*}{r_v} \Big(1-\sum_{f'\in E_{v,\neg f}} \gam_t \frac{x_{f'}^*}{r_v} \frac{p_{f'}}{2} \Big) \label{ineq:appen-1}\\
& \ge   \frac{x_f^*}{r_v} \Big(1-\frac{\gam_t}{2} \Big)
\end{align} 
Note that Inequality~\eqref{ineq:appen-1} is due to facts that $\E[H_{f'}| H_f=1] \le \gam_t$ from Lemma 3.1 in~\cite{BSSX17} and $\E[X^v_{f'} |X^v_f=1] \le x_{f'}^*/r_v$ from negative correlation in dependent rounding. Similarly we can prove for the case of $\sr(\y^v)$. 
\end{proof}

Lemma~\ref{lem:app-1} justifies the edge-attenuation steps~\eqref{alg:att-1} and~\eqref{alg:att-2} in $\alg(\alp, \beta)$. The lemma below will instead justify the vertex-attenuation step~\eqref{alg:att-3} in $\alg(\alp, \beta)$.

\begin{lemma}\label{lem:app-2}
Consider a given  $u \in U$ and assume that $u$ is available at $t$ with probability equal to $\gam_t$. We have that $u$ is available at $t+1$ with probability at least $\gam_t(1-\mu_t/T)$ before any vertex-attenuation during $t+1$. 
\end{lemma}

\begin{proof}
Assume that $u$ is available at $t$ (\ie $\SF_{u,t}$), which occurs with probability $\gam_t$. Now we have that $u$ survives from round $t$ if none of $f \in E_u$ gets matched.  Thus, we have that
\begin{align}
\Pr[\SF_{u,t+1}]&=\Pr[\SF_{u,t}] \Big(1-\sum_{f=(u,v) \in E_u}\Pr[ \mbox{$f$ is matched}~|~\SF_{u,t}] \Big) \label{ineq:app-2-1}\\
&=\gam_t \left(1-\sum_{f=(u,v) \in E_u} \frac{r_v}{T} \Big( \frac{\alp \mu_t x_f^*}{r_v}+ \frac{\beta \mu_t y_f^*}{r_v} \Big) p_f\right) \label{ineq:app-2-2}\\
&=\gam_t \left(1-\sum_{f=(u,v) \in E_u} \frac{\mu_t}{T} \Big( \alp  x_f^*p_f+ \beta  y_f^* p_f\Big)\right) \label{ineq:app-2-3}\\
& \ge \gam_t (1- \mu_t/T)  \label{ineq:app-2-4}
\end{align} 
Equality~\eqref{ineq:app-2-3} is due to the fact that each $f \in E_{v,t}$ will be probed with probability equal to $\mu_t x_f^*/r_v $ and $\mu_t y_f^*/r_v$ in $\sr(\x^v)$ and $\sr(\y^v)$ respectively after edge-attenuation in $\alg(\alp, \beta)$. Inequality~\eqref{ineq:app-2-4} follows from facts that (1) $\sum_{f \in E_u} x_f^* p_f \le B_u=1$ and $\sum_{f \in E_u} y_f^* p_f \le B_u=1$ due to Constraint~\eqref{cons:match-u} in benchmark LPs; (2) $\alp+\beta \le 1$. 
\end{proof}

Lemmas~\ref{lem:app-1} and~\ref{lem:app-2} justify $\alg(\alp, \beta)$. Now we try to prove main Theorem \ref{thm:main-2}.

\begin{proof}
Consider a given $f =(u,v)$. Let $\kap^{\x}_f$ and $\kap^{\y}_f$ be the number of probed of $f$ in $\sr(\x^v)$ and $\sr(\y^v)$. Thus, we have
\begin{align}
\kap^{\x}_f &=\sum_{t=1}^T \Pr[\SF_{u,t}]\frac{r_v}{T} \alp \frac{\mu_t x_f^*}{r_v}=\alp x_f^* \sum_{t=1}^T \frac{1}{T} \mu_t \gam_t 
\end{align}
According to the definition of $\{\gam_t, \mu_t\}$,
we can verify that $\sum_{t=1}^T \frac{1}{T} \mu_t \gam_t =\frac{e-1}{e+1}$ when $T \rightarrow \infty$. Thus, we claim that $\kap^{\x}_f=\alp x_f^* \frac{e-1}{e+1}$. Similarly, we have $\kap^{\y}_f=\alp y_f^* \frac{e-1}{e+1}$. By linearity of expectation, we get our claim.
\end{proof}

	\end{document}